\documentclass{llncs}
\usepackage{amsmath}
\usepackage{amssymb}
\usepackage{amsfonts}
\usepackage{makeidx}

\input{tcilatex}
\begin{document}

\frontmatter
\author{Andreas Maurer}

\mainmatter

\title{A chain rule for the expected suprema of Gaussian processes}%

\titlerunning{Chain rule} 

\authorrunning{A. Maurer}

\institute{Adalbertstrasse 55\\D-80799 M\"unchen, Germany\\
am@andreas-maurer.eu
}

\maketitle%

\begin{abstract}%
The expected supremum of a Gaussian process indexed by the image of an index
set under a function class is bounded in terms of separate properties of the
index set and the function class. The bound is relevant to the estimation of
nonlinear transformations or the analysis of learning algorithms whenever
hypotheses are chosen from composite classes, as \ is the case for
multi-layer models.%
\end{abstract}%

\section{Introduction}

Rademacher and Gaussian averages (\cite{Bartlett 2002}, see also \cite%
{Kolchinskii 2000},\cite{Mendelson 2001}) provide an elegant method to
demonstrate generalization for a wide variety of learning algorithms and are
particularly well suited to analyze kernel machines, where the use of more
classical methods relying on covering numbers becomes cumbersome.

To briefly describe the use of Gaussian averages (Rademacher averages will
not concern us), let $Y\subseteq 
\mathbb{R}
^{n}$ and let $\mathbf{\gamma }$ be a vector $\mathbf{\gamma }=\left( \gamma
_{1},...,\gamma _{n}\right) $ of independent standard normal variables. We
define the (expected supremum of the) Gaussian average of $Y$ as%
\begin{equation}
G\left( Y\right) =\mathbb{E}\sup_{\mathbf{y}\in Y}\left\langle \mathbf{%
\gamma },\mathbf{y}\right\rangle ,  \label{Gaussian average 1st definition}
\end{equation}%
where $\left\langle .,.\right\rangle $ denotes the inner product in $%
\mathbb{R}
^{n}$. Consider a loss class $\tciFourier $ of functions $f:\mathcal{X}%
\rightarrow 
\mathbb{R}
$, where $\mathcal{X}$ is some space of examples (such as input-output
pairs), a sample $\mathbf{x}=\left( x_{1},...,x_{n}\right) \in \mathcal{X}%
^{n}$ of observations and write $\tciFourier \left( \mathbf{x}\right) $ for
the subset of $%
\mathbb{R}
^{n}$ given by $\tciFourier \left( \mathbf{x}\right) =\left\{ \left( f\left(
x_{1}\right) ,...,f\left( x_{n}\right) \right) :f\in \tciFourier \right\} $.
Then we have the following result \cite{Bartlett 2002}.

\begin{theorem}
\label{Theorem Rademacher bound}Let the members of $\tciFourier $ take
values in $\left[ 0,1\right] $ and let $X,X_{1},...,X_{n}$ be iid random
variables with values in $\mathcal{X}$, $\mathbf{X}=\left(
X_{1},...,X_{n}\right) $. Then for $\delta >0$ with probability at least $%
1-\delta $ we have for every $f\in \tciFourier $ that%
\begin{equation*}
\mathbb{E}f\left( X\right) \leq \frac{1}{n}\sum f\left( X_{i}\right) +\frac{%
\sqrt{2\pi }}{n}G\left( \tciFourier \left( \mathbf{X}\right) \right) +\sqrt{%
\frac{9\ln 2/\delta }{2n}},
\end{equation*}%
where the expectation in the definition (\ref{Gaussian average 1st
definition}) of $G\left( \tciFourier \left( \mathbf{X}\right) \right) $ is
conditional to the sample $\mathbf{X}$.\bigskip 
\end{theorem}

The utility of Gaussian averages is not limited to functions with values in $%
\left[ 0,1\right] $. For real functions $\phi $ with Lipschitz constant $%
L\left( \phi \right) $\ we have $G\left( \left( \phi \circ \tciFourier
\right) \left( \mathbf{x}\right) \right) $ $\leq L\left( \phi \right)
~G\left( \tciFourier \left( \mathbf{x}\right) \right) $ (see also Slepian's
Lemma, \cite{Ledoux 1991}, \cite{Boucheron 2013}), where $\phi \circ
\tciFourier $ is the class $\left\{ x\mapsto \phi \left( f\left( x\right)
\right) :f\in \tciFourier \right\} $.\bigskip 

The inequality $G\left( \left( \phi \circ \tciFourier \right) \left( \mathbf{%
x}\right) \right) \leq L\left( \phi \right) ~G\left( \tciFourier \left( 
\mathbf{x}\right) \right) $, which in the above form holds also for
Rademacher averages \cite{Meir Zhang}, is extremely useful and in part
responsible for the success of these complexity measures. For Gaussian
averages it holds in a more general sense: if $\phi :%
\mathbb{R}
^{n}\rightarrow 
\mathbb{R}
^{m}$ has Lipschitz constant $L\left( \phi \right) $ with respect to the
Euclidean distances, then $G\left( \phi \left( Y\right) \right) \leq L\left(
\phi \right) G\left( Y\right) $. This is a direct consequence of Slepian's
Lemma and can be applied to the analysis of clustering or learning to learn (%
\cite{MP-ieee} and \cite{Maurer 2009}).

But what if we also want some freedom in the choice of $\phi $\textit{\ after%
} seeing the data? If the class of Lipschitz functions considered has small
cardinality, a union bound can be used. If it is very large one can try to
use covering numbers, but the matter soon becomes quite complicated and
destroys the elegant simplicity of the method.\bigskip

These considerations lead to a more general question: given a set $Y\subset 
\mathbb{R}
^{n}$ and a class $\tciFourier $ of functions $f:\mathcal{%
\mathbb{R}
}^{n}\mathcal{\rightarrow 
\mathbb{R}
}^{m}$, how can we bound the Gaussian average $G\left( \tciFourier \left(
Y\right) \right) =G\left( \left\{ f\left( y\right) :f\in \tciFourier ,y\in
Y\right\} \right) $ in terms of separate properties of $Y$ and $\tciFourier $%
, properties which should preferably very closely resemble Gaussian
averages? If $\mathcal{H}$ is some class of functions mapping samples to $%
\mathbb{R}
^{n}$ and $Y=\mathcal{H}\left( \mathbf{x}\right) $, then the bound is on $%
G\left( \tciFourier \left( Y\right) \right) =G\left( \left( \tciFourier
\circ \mathcal{H}\right) \left( \mathbf{x}\right) \right) $, so our question
is relevant to the estimation of composite functions in general. Such
estimates are necessary for multitask feature-learning, where $\mathcal{H}$
is a class of feature maps and $\tciFourier $ is vector-valued, with
components chosen independently for each task. Other potential applications
are to the currently popular subject of deep learning, where we consider
functional concatenations as in $\mathcal{F}_{M}\mathcal{\circ F}_{M-1}%
\mathcal{\circ }...\circ \mathcal{F}_{1}$. \bigskip

The present paper gives a preliminary answer. To state it we introduce some
notation. We will always take $\mathbf{\gamma }=\left( \gamma
_{1},...\right) $ to be a random vector whose components are independent
standard normal variables, while $\left\Vert .\right\Vert $ and $%
\left\langle .,.\right\rangle $ denote norm and inner product in a Euclidean
space, the dimension of which is determined by context, as is the dimension
of the vector $\mathbf{\gamma }$.

\begin{definition}
If $Y\subseteq 
\mathbb{R}
^{n}$ we set%
\begin{equation*}
D\left( Y\right) =\sup_{\mathbf{y},\mathbf{y}^{\prime }\in Y}\left\Vert 
\mathbf{y}-\mathbf{y}^{\prime }\right\Vert \text{ and }G\left( Y\right) =%
\mathbb{E}\sup_{\mathbf{y}\in Y}\left\langle \mathbf{\gamma },\mathbf{y}%
\right\rangle .
\end{equation*}%
If $\tciFourier $ is a class of functions $f:Y\rightarrow 
\mathbb{R}
^{m}$ we set%
\begin{eqnarray*}
L\left( \tciFourier ,Y\right) &=&\sup_{\mathbf{y},\mathbf{y}^{\prime }\in Y,~%
\mathbf{y}\neq \mathbf{y}^{\prime }}\sup_{f\in \tciFourier }\frac{\left\Vert
f\left( \mathbf{y}\right) -f\left( \mathbf{y}^{\prime }\right) \right\Vert }{%
\left\Vert \mathbf{y}-\mathbf{y}^{\prime }\right\Vert }\text{ and} \\
R\left( \tciFourier ,Y\right) &=&\sup_{\mathbf{y},\mathbf{y}^{\prime }\in Y,~%
\mathbf{y}\neq \mathbf{y}^{\prime }}\mathbb{E}\sup_{f\in \tciFourier }\frac{%
\left\langle \mathbf{\gamma },f\left( \mathbf{y}\right) -f\left( \mathbf{y}%
^{\prime }\right) \right\rangle }{\left\Vert \mathbf{y}-\mathbf{y}^{\prime
}\right\Vert }.
\end{eqnarray*}%
We also write $\tciFourier \left( Y\right) =\left\{ f\left( \mathbf{y}%
\right) :f\in \tciFourier ,\mathbf{y}\in Y\right\} $. When there is no
ambiguity we write $L\left( \tciFourier \right) =L\left( \tciFourier
,Y\right) $ and $R\left( \tciFourier \right) =R\left( \tciFourier ,Y\right) $%
.\bigskip
\end{definition}

Then $D\left( Y\right) $ is the diameter of $Y$, and $G\left( Y\right) $ is
the Gaussian average already introduced above. $L\left( \tciFourier \right) $
is the smallest Lipschitz constant acceptable for all $f\in \tciFourier $,
and the more unusual quantity $R\left( \tciFourier \right) $ can be viewed
as a Gaussian average of Lipschitz quotients. In section \ref{Section some
properties of R} we give some properties of $R\left( \tciFourier \right) $.
Our main result is the following chain rule.\bigskip

\begin{theorem}
\label{Theorem Main Main}Let $Y\subset 
\mathbb{R}
^{n}$ be finite, $\tciFourier $ a finite class of functions $f:Y\rightarrow 
\mathbb{R}
^{m}$. Then there are universal constants $C_{1}$ and $C_{2}$ such that for
any $\mathbf{y}_{0}\in Y$%
\begin{equation}
G\left( \tciFourier \left( Y\right) \right) \leq C_{1}L\left( \tciFourier
\right) G\left( Y\right) +C_{2}D\left( Y\right) R\left( \tciFourier \right)
+G\left( \tciFourier \left( \mathbf{y}_{0}\right) \right) .
\label{Main inequality}
\end{equation}
\end{theorem}

We make some general remarks on the implications of our result.

1. The requirement of finiteness for $Y$ and $\tciFourier $ is a
simplification to avoid issues of measurability. The cardinality of these
sets plays no role.

2. The constants $C_{1}$ and $C_{2}$ as they result from the proof are
rather large, because they accumulate the constants of Talagrand's
majorizing measure theorem and generic chaining \cite{Ledoux 1991}\cite%
{Talagrand 1992}\cite{Talagrand 2001}\cite{Talagrand 2005}. This is a major
shortcoming and the reason why our result is regarded as preliminary. Is
there another proof of a similar result, avoiding majorizing measures and
resulting in smaller constants? This question is the subject of current
research.

3. The first term on the right hand side of (\ref{Main inequality})
describes the complexity inherited from the bottom layer $Y$ (which we may
think of as $\mathcal{H}\left( \mathbf{x}\right) $), and it depends on the
top layer $\tciFourier $ only through the Lipschitz constant $L\left(
\tciFourier \right) $. The other two terms represent the complexity of the
top layer, depending on the bottom layer only through the diameter $D\left(
Y\right) $ of $Y$. If $Y$ has unit diameter and the functions in $%
\tciFourier $ are contractions, then the two layers are completely decoupled
in the bound. This decoupling is the most attractive property of our result.

4. Apart from the large constants the inequality is tight in at least two
situations: first, if $Y=\left\{ \mathbf{y}_{0}\right\} $ is a singleton,
then only the last term remains, and we recover the Gaussian average of $%
\mathcal{\tciFourier }\left( \mathbf{y}_{0}\right) $. This also shows that
the last term cannot be eliminated. On the other hand if $\tciFourier $
consists of a single Lipschitz function $\phi $, then we recover (up to a
constant) the inequality $G\left( \phi \left( Y\right) \right) \leq L\left(
\phi \right) G\left( Y\right) $ above.

5. The bound can be iterated to multiple layers by re-substitution of $%
\tciFourier \left( Y\right) $ in place of $Y$. A corresponding formula is
given in Section \ref{Section applications}, where we also sketch
applications to vector-valued function classes.

The next section gives a proof of Theorem \ref{Theorem Main Main}, then we
explain how our result can be applied to machine learning. The last section
is devoted to the proof of a technical result encapsulating our use of
majorizing measures.

\section{Proving the chain rule}

To prove Theorem \ref{Theorem Main Main} we need the theory of majorizing
measures and generic chaining. Our use of these techniques is summarized in
the following theorem, which is also the origin of our large constants.

\begin{theorem}
\label{Theorem generic chaining}Let $X_{\mathbf{y}}$ be a random process
indexed by a finite set $Y\subset 
\mathbb{R}
^{n}$. Suppose that there is a number $K\geq 1$ such that for any distinct
members $\mathbf{y},\mathbf{y}^{\prime }\in Y$ and any $s>0$ 
\begin{equation}
\Pr \left\{ X_{\mathbf{y}}-X_{\mathbf{y}^{\prime }}>s\right\} \leq K\exp
\left( \frac{-s^{2}}{2\left\Vert \mathbf{y}-\mathbf{y}^{\prime }\right\Vert
^{2}}\right)  \label{Process Tailbound}
\end{equation}
Then for any $\mathbf{y}_{0}\in Y$%
\begin{equation*}
\mathbb{E}\left[ \sup_{\mathbf{y}\in Y}X_{\mathbf{y}}-X_{\mathbf{y}_{0}}%
\right] \leq C^{\prime }G\left( Y\right) +C^{\prime \prime }D\left( Y\right) 
\sqrt{\ln K},
\end{equation*}%
where $C^{\prime }$ and $C^{\prime \prime }$ are universal constants.
\end{theorem}

This is obtained from Talagrand's majorizing measure theorem (Theorem \ref%
{Theorem majorizing measure} below) combined with generic chaining \cite%
{Talagrand 2005}. An early version of a similar result is Theorem 15 in \cite%
{Talagrand 1987}, where the author remarks that his method of proof (which
we also use) is very indirect, and that a more direct proof would be
desirable. In Section \ref{Section proof of generic chaining} we do supply a
proof, largely because the dependence on $K$, which can often be swept under
the carpet, plays a crucial role in our arguments below.

We also need the following Gaussian concentration inequality
(Tsirelson-Ibragimov-Sudakov inequality, Theorem 5.6 in \cite{Boucheron 2013}%
).

\begin{theorem}
\label{Theorem Gaussian concentration}Let $F:%
\mathbb{R}
^{n}\rightarrow 
\mathbb{R}
$ be $L$-Lipschitz. Then for any $s>0$%
\begin{equation*}
\Pr \left\{ F\left( \mathbf{\gamma }\right) >{{\mathbb{E}}}F\left( \mathbf{%
\gamma }\right) +s\right\} \leq e^{-s^{2}/\left( 2L^{2}\right) }.
\end{equation*}
\end{theorem}

To conclude the preparation for the proof of Theorem \ref{Theorem Main Main}
we give a simple lemma.

\begin{lemma}
\label{Lemma calculus}Suppose a random variable $X$ satisfies $\Pr \left\{
X-A>s\right\} \leq e^{-s^{2}}$, for any $s>0$. Then%
\begin{equation*}
\left. \forall s>0\right. ,\text{ }\Pr \left\{ X>s\right\} \leq
e^{A^{2}}e^{-s^{2}/2}.
\end{equation*}
\end{lemma}

\begin{proof}
For $s\leq A$ the conclusion is trivial, so suppose that $s>A$. From $%
s^{2}=\left( s-A+A\right) ^{2}\leq 2\left( s-A\right) ^{2}+2A^{2}$ we get $%
\left( s-A\right) ^{2}\geq \left( s^{2}/2\right) -A^{2}$, so%
\begin{equation*}
\Pr \left\{ X>s\right\} =\Pr \left\{ X-A>s-A\right\} \leq e^{-\left(
s-A\right) ^{2}}\leq e^{A^{2}}e^{-s^{2}/2}.
\end{equation*}%
$\blacksquare $\bigskip
\end{proof}

\begin{proof}[of Theorem \protect\ref{Theorem Main Main}]
The result is trivial if $\tciFourier $ consists only of constants, so we
can assume that $L\left( \tciFourier \right) >0$. For $\mathbf{y},\mathbf{y}%
^{\prime }\in Y$ define a function $F:%
\mathbb{R}
^{m}\rightarrow 
\mathbb{R}
$ by%
\begin{equation*}
F\left( \mathbf{z}\right) =\sup_{f\in \tciFourier }\left\langle \mathbf{z}%
,f\left( \mathbf{y}\right) -f\left( \mathbf{y}^{\prime }\right)
\right\rangle .
\end{equation*}%
$F$ is Lipschitz with Lipschitz constant bounded by $\sup_{f\in \tciFourier
}\left\Vert f\left( \mathbf{y}\right) -f\left( \mathbf{y}^{\prime }\right)
\right\Vert \leq L\left( \tciFourier \right) \left\Vert \mathbf{y}-\mathbf{y}%
^{\prime }\right\Vert $. Writing $Z_{\mathbf{y},\mathbf{y}^{\prime
}}=F\left( \mathbf{\gamma }\right) $, it then follows from Gaussian
concentration (Theorem \ref{Theorem Gaussian concentration}) that%
\begin{equation*}
\Pr \left\{ Z_{\mathbf{y},\mathbf{y}^{\prime }}>\mathbb{E}Z_{\mathbf{y},%
\mathbf{y}^{\prime }}+s\right\} \leq \exp \left( \frac{-s^{2}}{2L\left(
\tciFourier \right) ^{2}\left\Vert \mathbf{y}-\mathbf{y}^{\prime
}\right\Vert ^{2}}\right) .
\end{equation*}%
Since by definition $\mathbb{E}Z_{\mathbf{y},\mathbf{y}^{\prime }}\leq
R\left( \tciFourier \right) \left\Vert \mathbf{y}-\mathbf{y}^{\prime
}\right\Vert $, Lemma \ref{Lemma calculus} gives%
\begin{equation*}
\Pr \left\{ Z_{\mathbf{y},\mathbf{y}^{\prime }}>s\right\} \leq \exp \left( 
\frac{R\left( \tciFourier \right) ^{2}}{2L\left( \tciFourier \right) ^{2}}%
\right) \exp \left( \frac{-s^{2}}{4L\left( \tciFourier \right)
^{2}\left\Vert \mathbf{y}-\mathbf{y}^{\prime }\right\Vert ^{2}}\right) .
\end{equation*}%
Now define a process $X_{\mathbf{y}}$, indexed by $Y$, as%
\begin{equation*}
X_{\mathbf{y}}=\frac{1}{\sqrt{2}L\left( \tciFourier \right) }\sup_{f\in
\tciFourier }\left\langle \mathbf{\gamma },f\left( \mathbf{y}\right)
\right\rangle .
\end{equation*}%
Since $X_{\mathbf{y}}-X_{\mathbf{y}^{\prime }}\leq Z_{\mathbf{y},\mathbf{y}%
^{\prime }}/\left( \sqrt{2}L\left( \tciFourier \right) \right) $ we have 
\begin{eqnarray*}
\Pr \left\{ X_{\mathbf{y}}-X_{\mathbf{y}^{\prime }}>s\right\}  &\leq &\Pr
\left\{ Z_{\mathbf{y},\mathbf{y}^{\prime }}>\sqrt{2}L\left( \tciFourier
\right) s\right\}  \\
&\leq &\exp \left( \frac{R\left( \tciFourier \right) ^{2}}{2L\left(
\tciFourier \right) ^{2}}\right) \exp \left( \frac{-s^{2}}{2\left\Vert 
\mathbf{y}-\mathbf{y}^{\prime }\right\Vert ^{2}}\right) 
\end{eqnarray*}%
and by Theorem \ref{Theorem generic chaining}, with $K=\exp \left( R\left(
\tciFourier \right) ^{2}/\left( 2L\left( \tciFourier \right) ^{2}\right)
\right) \geq 1$,%
\begin{equation*}
\mathbb{E}\sup_{\mathbf{y}\in Y}\left( X_{\mathbf{y}}-X_{\mathbf{y}%
_{0}}\right) \leq C^{\prime }G\left( Y\right) +C^{\prime \prime }D\left(
Y\right) \frac{R\left( \tciFourier \right) }{\sqrt{2}L\left( \tciFourier
\right) }.
\end{equation*}%
Multiplication by $\sqrt{2}L\left( \tciFourier \right) $ then gives%
\begin{equation*}
\mathbb{E}\sup_{\mathbf{y}\in Y}\left( \sup_{f\in \tciFourier }\left\langle 
\mathbf{\gamma },f\left( \mathbf{y}\right) \right\rangle -\sup_{f\in
\tciFourier }\left\langle \mathbf{\gamma },f\left( \mathbf{y}_{0}\right)
\right\rangle \right) \leq C_{1}L\left( \tciFourier \right) G\left( Y\right)
+C_{2}D\left( Y\right) R\left( \tciFourier \right) 
\end{equation*}%
with $C_{1}=\sqrt{2}C^{\prime }$ and $C_{2}=C^{\prime \prime }$. $%
\blacksquare $
\end{proof}

\section{Applications\label{Section applications}}

We first give some elementary properties of the quantity $R\left(
\tciFourier ,Y\right) $ which appears in Theorem \ref{Theorem Main Main}.
Then we apply Theorem \ref{Theorem Main Main} to a two layer kernel machine
and give a bound for multi-task learning of low-dimensional representations.

\subsection{Some properties of $R\left( \tciFourier \right) $\label{Section
some properties of R}}

Recall the definition of $R\left( \tciFourier ,Y\right) $. If $Y\subseteq 
\mathbb{R}
^{n}$and $\tciFourier $ consists of functions $f:Y\rightarrow 
\mathbb{R}
^{m}$%
\begin{equation*}
R\left( \tciFourier ,Y\right) =\sup_{\mathbf{y},\mathbf{y}^{\prime }\in Y,~%
\mathbf{y}\neq \mathbf{y}^{\prime }}\mathbb{E}\sup_{f\in \tciFourier }\frac{%
\left\langle \mathbf{\gamma },f\left( \mathbf{y}\right) -f\left( \mathbf{y}%
^{\prime }\right) \right\rangle }{\left\Vert \mathbf{y}-\mathbf{y}^{\prime
}\right\Vert }.
\end{equation*}%
$R\left( \tciFourier \right) $ is itself a supremum of Gaussian averages.
For $\mathbf{y},\mathbf{y}^{\prime }\in Y$ let $\Delta \tciFourier \left( 
\mathbf{y},\mathbf{y}^{\prime }\right) \subseteq 
\mathbb{R}
^{m}$ be the set of quotients%
\begin{equation*}
\Delta \tciFourier \left( \mathbf{y},\mathbf{y}^{\prime }\right) =\left\{ 
\frac{f\left( \mathbf{y}\right) -f\left( \mathbf{y}^{\prime }\right) }{%
\left\Vert \mathbf{y}-\mathbf{y}^{\prime }\right\Vert }:f\in \tciFourier
\right\} .
\end{equation*}%
It follows from the definition that $R\left( \tciFourier ,Y\right) =\sup_{%
\mathbf{y},\mathbf{y}^{\prime }\in Y,~\mathbf{y}\neq \mathbf{y}^{\prime
}}G\left( \Delta \tciFourier \left( \mathbf{y},\mathbf{y}^{\prime }\right)
\right) $. We record some simple properties. Recall that for a set $S$ in a
real vector space the convex hull $Co\left( S\right) $ is defined as 
\begin{equation*}
Co\left( S\right) =\left\{ \sum_{i=1}^{n}\alpha _{i}z_{i}:n\in 
\mathbb{N}
\text{, }z_{i}\in S,\alpha _{i}\geq 0\text{, }\sum_{i}\alpha _{i}=1\right\} .
\end{equation*}

\begin{theorem}
Let $Y\subseteq 
\mathbb{R}
^{n}$ and let $\tciFourier $ and $\mathcal{H}$ be classes of functions $%
f:Y\rightarrow 
\mathbb{R}
^{m}$. Then

(i) If $\tciFourier \subseteq \mathcal{H}$ then $R\left( \tciFourier
,Y\right) \leq R\left( \mathcal{H},Y\right) $.

(ii) If $Y\subseteq Y^{\prime }$ then $R\left( \tciFourier ,Y\right) \leq
R\left( \tciFourier ,Y^{\prime }\right) $.

(iii) If $c\geq 0$ then $R\left( c\tciFourier ,Y\right) =cR\left(
\tciFourier ,Y\right) .$

(iv) $R\left( \tciFourier +\mathcal{H},Y\right) \leq R\left( \tciFourier
,Y\right) +R\left( \mathcal{H},Y\right) $.

(v) $R\left( \tciFourier ,Y\right) =R\left( Co\left( \tciFourier \right)
,Y\right) $.

(vi) If $Z\subseteq 
\mathbb{R}
^{K}$ and $\phi :Z\rightarrow 
\mathbb{R}
^{n}$ has Lipschitz constant $L\left( \phi \right) $ and the members of $%
\tciFourier $ are defined on $\phi \left( Z\right) $, then $R\left(
\tciFourier \circ \phi ,Z\right) \leq L\left( \phi \right) R\left(
\tciFourier ,\phi \left( Z\right) \right) $.

(vii) $R\left( \tciFourier \right) \leq L\left( \tciFourier \right) \sqrt{%
2\ln \left\vert \tciFourier \right\vert }.$
\end{theorem}

Remarks:

1. From (ii) we get $R\left( \tciFourier ,Y\right) \leq R\left( \tciFourier ,%
\mathbb{R}
^{n}\right) $. In applications where $Y=\mathcal{H}\left( \mathbf{x}\right) $
the quantity $R\left( \tciFourier ,\mathcal{H}\left( \mathbf{x}\right)
\right) $ is data-dependent, but $R\left( \tciFourier ,%
\mathbb{R}
^{n}\right) $ is sometimes easier to bound.

2. We see that the properties of $R\left( \tciFourier \right) $ largely
parallel the properties of the Gaussian averages themselves, except for the
inequality $G\left( \phi \left( Y\right) \right) \leq L\left( \phi \right)
G\left( Y\right) $, for which there doesn't seem to be an analogous property
of $R\left( \tciFourier \right) $. Instead we have a 'backwards' version of
it with (vi) above, with a rather trivial proof below.

3. Of course (vii) is relevant only when $\ln \left\vert \tciFourier
\right\vert $ is reasonably small and serves the comparison of Theorem \ref%
{Theorem Main Main} to alternative bounds.

\begin{proof}
(i)-(iii) are obvious from the definition. (iv) follows from linearity of
the inner product and the triangle inequality for the supremum. To see (v)
first note that $R\left( \tciFourier \right) \leq R\left( Co\left(
\tciFourier \right) \right) $ follows from (i), while the reverse inequality
follows from%
\begin{eqnarray*}
&&\sup_{\alpha _{i}\geq 0,\sum \alpha _{i}=1}\sup_{f_{1},f_{2},...\in
\tciFourier }\left\langle \mathbf{\gamma },\sum_{i}\alpha _{i}f_{i}\left( 
\mathbf{y}\right) -\sum_{i}\alpha _{i}f_{i}\left( \mathbf{y}^{\prime
}\right) \right\rangle \\
&=&\sup_{\alpha _{i}\geq 0,\sum \alpha _{i}=1}\sup_{f_{1},f_{2},...\in
\tciFourier }\sum_{i}\alpha _{i}\left\langle \mathbf{\gamma },f_{i}\left( 
\mathbf{y}\right) -f_{i}\left( \mathbf{y}^{\prime }\right) \right\rangle \\
&\leq &\sup_{\alpha _{i}\geq 0,\sum \alpha _{i}=1}\sum_{i}\alpha
_{i}\sup_{f\in \tciFourier }\left\langle \mathbf{\gamma },f\left( \mathbf{y}%
\right) -f\left( \mathbf{y}^{\prime }\right) \right\rangle \\
&=&\sup_{f\in \tciFourier }\left\langle \mathbf{\gamma },f\left( \mathbf{y}%
\right) -f\left( \mathbf{y}^{\prime }\right) \right\rangle .
\end{eqnarray*}%
For (vi) we may chose $\mathbf{y}$ and $\mathbf{y}^{\prime }$ such that $%
\phi \left( \mathbf{y}\right) \neq \phi \left( \mathbf{y}^{\prime }\right) $%
, since otherwise both sides of the inequality to be proved are zero. But
then%
\begin{eqnarray*}
\mathbb{E}\sup_{f\in \tciFourier \circ \phi }\frac{\left\langle \mathbf{%
\gamma },f\left( \mathbf{y}\right) -f\left( \mathbf{y}^{\prime }\right)
\right\rangle }{\left\Vert \mathbf{y}-\mathbf{y}^{\prime }\right\Vert } &=&%
\frac{\left\Vert \phi \left( \mathbf{y}\right) -\phi \left( \mathbf{y}%
^{\prime }\right) \right\Vert }{\left\Vert \mathbf{y}-\mathbf{y}^{\prime
}\right\Vert }\mathbb{E}\sup_{f\in \tciFourier }\frac{\left\langle \mathbf{%
\gamma },f\left( \phi \left( \mathbf{y}\right) \right) -f\left( \phi \left( 
\mathbf{y}^{\prime }\right) \right) \right\rangle }{\left\Vert \phi \left( 
\mathbf{y}\right) -\phi \left( \mathbf{y}^{\prime }\right) \right\Vert } \\
&\leq &L\left( \phi \right) \mathbb{E}\sup_{f\in \tciFourier }\frac{%
\left\langle \mathbf{\gamma },f\left( \phi \left( \mathbf{y}\right) \right)
-f\left( \phi \left( \mathbf{y}^{\prime }\right) \right) \right\rangle }{%
\left\Vert \phi \left( \mathbf{y}\right) -\phi \left( \mathbf{y}^{\prime
}\right) \right\Vert }.
\end{eqnarray*}%
To see (vii) note that for every $\mathbf{y}$ and $\mathbf{y}^{\prime }$ and
every $f\in \tciFourier $ it follows from Gaussian concentration (Theorem %
\ref{Theorem Gaussian concentration}) that 
\begin{equation*}
\Pr \left\{ \frac{\left\langle \mathbf{\gamma },f\left( \mathbf{y}\right)
-f\left( \mathbf{y}^{\prime }\right) \right\rangle }{\left\Vert \mathbf{y}-%
\mathbf{y}^{\prime }\right\Vert }>s\right\} \leq e^{-s^{2}/2L^{2}}.
\end{equation*}%
The conclusion then follows from standard estimates (e.g. \cite{Boucheron
2013}, section 2.5). $\blacksquare $\bigskip
\end{proof}

\subsection{A double layer kernel machine}

We use the chain rule to bound the complexity of a double-layer kernel
machine. The corresponding optimization problem is clearly non-convex and we
are not aware of an efficient optimization method. The model is chosen to
illustrate the application of Theorem \ref{Theorem Main Main}. It is defined
as follows.

Assume the data to lie in $%
\mathbb{R}
^{m_{0}}$ and fix two real numbers $\Delta _{1}$ and $B_{1}$. On $%
\mathbb{R}
^{m_{0}}\times 
\mathbb{R}
^{m_{0}}$ define a (Gaussian radial-basis-function) kernel $\kappa $ by 
\begin{equation*}
\kappa \left( z,z^{\prime }\right) =\exp \left( \frac{-\left\Vert
z-z^{\prime }\right\Vert ^{2}}{2\Delta _{1}^{2}}\right) ,~z,z^{\prime }\in 
\mathbb{R}
^{m_{0}},
\end{equation*}%
and let $\phi :%
\mathbb{R}
^{m_{0}}\rightarrow H$ be the associated feature map, where $H$ is the
associated RKHS with inner product $\left\langle .,.\right\rangle _{H}$ and
norm $\left\Vert .\right\Vert _{H}$ (for kernel methods see . Now we let $%
\mathcal{H}$ be the class of vector valued functions $h:%
\mathbb{R}
^{m_{0}}\rightarrow 
\mathbb{R}
^{m_{1}}$ defined by 
\begin{equation*}
\mathcal{H}=\left\{ z\in 
\mathbb{R}
^{m_{0}}\mapsto \left( \left\langle w_{1},\phi \left( z\right) \right\rangle
_{H},...,\left\langle w_{m_{1}},\phi \left( z\right) \right\rangle
_{H}\right) :\sum_{k}\left\Vert w_{k}\right\Vert _{H}^{2}\leq
B_{1}^{2}\right\} .
\end{equation*}%
This can also be written as $\mathcal{H}=\left\{ z\in 
\mathbb{R}
^{m_{0}}\mapsto W\phi \left( z\right) :\left\Vert W\right\Vert _{HS}\leq
B_{1}\right\} $, where $\left\Vert W\right\Vert _{HS}$ is the
Hilbert-Schmidt norm of an operator $W:H\rightarrow 
\mathbb{R}
^{m_{1}}$.

For the function class $\mathcal{F}$, which we wish to compose with $%
\mathcal{H}$, we proceed in a similar way, defining an analogous kernel of
width $\Delta _{2}$ on $%
\mathbb{R}
^{m_{1}}\times 
\mathbb{R}
^{m_{1}}$, a corresponding feature map $\psi :%
\mathbb{R}
^{m_{1}}\rightarrow H$ and a class of real valued functions%
\begin{equation*}
\mathcal{F}=\left\{ z\in 
\mathbb{R}
^{m_{1}}\mapsto \left\langle v,\psi \left( z\right) \right\rangle
_{H}:\left\Vert v_{l}\right\Vert _{H}\leq B_{2}\right\} .
\end{equation*}%
We now want high probability bounds on the estimation error for functional
compositions $f\circ h$, uniform over $\tciFourier \circ \mathcal{H}$. To
apply our result we should really restrict to finite subsets of $\tciFourier 
$ and $\mathcal{H}$ a requirement which we simply ignore. In machine
learning we could of course always restrict all representations to some
fixed, very high but finite precision.

Fix a sample $\mathbf{x}\in 
\mathbb{R}
^{nm_{0}}$. Then $Y=\mathcal{H}\left( \mathbf{x}\right) \subset 
\mathbb{R}
^{nm_{1}}$. To use Theorem \ref{Theorem Main Main} we define a class $%
\tciFourier ^{\prime }$ of functions from $%
\mathbb{R}
^{nm_{1}}$ to $%
\mathbb{R}
^{n}$ by 
\begin{equation*}
\tciFourier ^{\prime }=\left\{ \left( y_{1},...,y_{n}\right) \in 
\mathbb{R}
^{nm_{1}}\mapsto \left( f\left( y_{1}\right) ,...,f\left( y_{n}\right)
\right) \in 
\mathbb{R}
^{n}:f\in \tciFourier \right\} \text{.}
\end{equation*}%
Since the first feature map $\phi $ maps to the unit sphere of $H$ we have 
\begin{eqnarray*}
D\left( \mathcal{H}\left( \mathbf{x}\right) \right) &\leq &2B_{1}\sqrt{n}%
\text{ and} \\
G\left( \mathcal{H}\left( \mathbf{x}\right) \right) &=&\mathbb{E}%
\sup_{W}\sum_{ik}\gamma _{ik}\left\langle w_{k},\phi \left( x_{i}\right)
\right\rangle _{H}\leq B_{1}\sqrt{nm_{1}}\text{.}
\end{eqnarray*}%
The feature map corresponding to the Gaussian kernel $\Delta _{2}$ has
Lipschitz constant $\Delta _{2}^{-1}$. For $\mathbf{y},\mathbf{y}^{\prime
}\in 
\mathbb{R}
^{nm_{1}}$ we obtain%
\begin{eqnarray*}
\sup_{v}\left( \sum_{i}\left( \left\langle v,\phi \left( y_{i}\right)
\right\rangle _{H}-\left\langle v,\phi \left( y_{i}^{\prime }\right)
\right\rangle _{H}\right) ^{2}\right) ^{1/2} &\leq &B_{2}\left(
\sum_{i}\left\Vert \phi \left( y_{i}\right) -\phi \left( y_{i}^{\prime
}\right) \right\Vert _{H}^{2}\right) ^{1/2} \\
&\leq &B_{2}\Delta _{2}^{-1}\left\Vert \mathbf{y}-\mathbf{y}^{\prime
}\right\Vert ,
\end{eqnarray*}%
so we have $L\left( \tciFourier ^{\prime },%
\mathbb{R}
^{nm_{1}}\right) \leq B_{2}\Delta _{2}^{-1}$.

On the other hand%
\begin{eqnarray*}
\mathbb{E}\sup_{v}\sum_{i}\gamma _{i}\left( \left\langle v,\phi \left(
y_{i}\right) \right\rangle _{H}-\left\langle v,\phi \left( y_{i}^{\prime
}\right) \right\rangle _{H}\right) &\leq &B_{2}\mathbb{E}\left\Vert
\sum_{i=1}^{n}\gamma _{i}\left( \phi \left( \mathbf{y}_{i}\right) -\phi
\left( \mathbf{y}_{i}^{\prime }\right) \right) \right\Vert \\
&\leq &B_{2}\left( \sum_{i}\left\Vert \phi \left( y_{i}\right) -\phi \left(
y_{i}^{\prime }\right) \right\Vert _{H}^{2}\right) ^{1/2} \\
&\leq &B_{2}\Delta _{2}^{-1}\left\Vert \mathbf{y}-\mathbf{y}^{\prime
}\right\Vert ,
\end{eqnarray*}%
so we have $R\left( \tciFourier ^{\prime },%
\mathbb{R}
^{nm_{1}}\right) \leq B_{2}\Delta _{2}^{-1}$. Furthermore 
\begin{equation*}
G\left( \tciFourier ^{\prime }\left( h_{0}\left( \mathbf{x}\right) \right)
\right) \leq B_{2}\sqrt{n},
\end{equation*}%
similar to the bound for $G\left( \mathcal{H}\left( \mathbf{x}\right)
\right) $.

For the composite network Theorem \ref{Theorem Main Main} gives us the bound%
\begin{equation*}
G\left( \tciFourier ^{\prime }\left( \mathcal{H}\left( \mathbf{x}\right)
\right) \right) \leq C_{1}B_{1}B_{2}\Delta _{2}^{-1}\sqrt{nm_{1}}%
+2C_{2}B_{1}B_{2}\sqrt{n}\Delta _{2}^{-1}+B_{2}\sqrt{n}.
\end{equation*}%
Dividing by $n$ and appealing to Theorem \ref{Theorem Rademacher bound} one
obtains the uniform bound: with probability at least $1-\delta $ we have for
every $h\in \mathcal{H}$ and every $f\in \tciFourier $ that%
\begin{eqnarray*}
\mathbb{E}f\left( h\left( X\right) \right)  &\leq &\frac{1}{n}\sum f\left(
h\left( X_{i}\right) \right) + \\
&&+\sqrt{\frac{2\pi }{n}}B_{2}\left( B_{1}\Delta _{2}^{-1}\left( C_{1}\sqrt{%
m_{1}}+2C_{2}\right) +1\right) +\sqrt{\frac{9\ln 2/\delta }{2n}}.
\end{eqnarray*}

Remarks.

1. One might object that the result depends heavily on the intermediate
dimension $m_{1}$ so that only a very classical relationship between sample
size and dimension is obtained. In this sense our result only works for
intermediate representations of rather low dimension. The mapping stages of $%
\mathcal{H}$ and $\tciFourier $ however include nonlinear maps to infinite
dimensional spaces.

2. Clearly the above choice of the Gaussian kernel is arbitrary. Any
positive semidefinite kernel can be used for the first mapping stage, and
the application of the chain rule requires only the Lipschitz property for
the second kernel in the definition of $\tciFourier $. The Gaussian kernel
was only chosen for definiteness.

3. Similarly the choice of the Hilbert-Schmidt norm as a regularizer for $W$
in the first mapping stage is arbitrary, one could equally use another
matrix norm. This would result in different bounds for $G\left( \mathcal{H}%
\left( \mathbf{x}\right) \right) $ and $D\left( \mathcal{H}\left( \mathbf{x}%
\right) \right) $, incurring a different dependency of our bound on $m_{1}$.

\subsection{Multitask learning}

As a second illustration we modify the above model to accommodate multitask
learning \cite{Baxter 1998}\cite{Baxter 2000}. Here one observes a $T\times n
$ sample $\mathbf{x=}\left( x_{ti}:1\leq t\leq T,1\leq i\leq n\right) $ $\in 
\mathcal{X}^{nT}$, where $\left( x_{ti}:1\leq i\leq n\right) $ is the sample
observed for the $t$-th task. We consider a two layer situation where the
bottom-layer $\mathcal{H}$ consists of functions $h:\mathcal{X\rightarrow 
\mathbb{R}
}^{m}$, and the top layer function class is of the form 
\begin{equation*}
\tciFourier ^{T}=\left\{ x\in 
\mathbb{R}
^{m_{1}}\mapsto \mathbf{f}\left( x\right) =\left( f_{1}\left( x\right)
,...,f_{T}\left( x\right) \right) \in 
\mathbb{R}
^{T}:f_{t}\in \tciFourier \right\} ,
\end{equation*}%
where $\tciFourier $ is some class of functions mapping $%
\mathbb{R}
^{m_{1}}$ to $%
\mathbb{R}
$. The functions (or representations) of the bottom layer $\mathcal{H}$ are
optimized for the entire sample, in the top layer each function $f_{t}$ is
optimized for the represented data corresponding to the $t$-th task. In an
approach of empirical risk minimization one selects the composed function $%
\mathbf{\hat{f}}\circ \hat{h}$ which minimizes the task-averaged empirical
loss%
\begin{equation*}
\min_{\mathbf{f}\in \tciFourier ^{n},h\in \mathcal{H}}\frac{1}{nT}%
\sum_{i=1}^{n}\sum_{t=1}^{T}f_{t}\left( h\left( x_{it}\right) \right) .
\end{equation*}%
We wish to give a general explanation of the potential benefits of this
method over the separate learning of functions from $\tciFourier \circ 
\mathcal{H}$, as studied in the previous section. Clearly we must assume
that the tasks are related in the sense that the above minimum is small, so
any possible benefit can only be a benefit of improved estimation.

For the multitask model a result analogous to Theorem \ref{Theorem
Rademacher bound} is easily obtained (see e.g. \cite{Maurer 2006}). Let $%
\mathbf{X=}\left( X_{ti}\right) $ be a vector of independent random
variables with values in $\mathcal{X}$, where $X_{ti}$ is iid to $X_{tj}$
for all $ijt$, and let $X_{t}$ be iid to $X_{ti}$. Then with probability at
least $1-\delta $ we have for every $\mathbf{f}\in \tciFourier ^{n}$ and
every $h\in \mathcal{H}$%
\begin{equation*}
\frac{1}{T}\sum_{t}\mathbb{E}f_{t}\left( h\left( X_{t}\right) \right) \leq 
\frac{1}{nT}\sum_{ti}f_{t}\left( h\left( X_{ti}\right) \right) +\frac{\sqrt{%
2\pi }}{nT}G\left( \tciFourier ^{T}\circ \mathcal{H}\left( \mathbf{X}\right)
\right) +\sqrt{\frac{9\ln 2/\delta }{2nT}}.
\end{equation*}%
Here the left hand side is interpreted as the task averaged risk and 
\begin{equation*}
G\left( \tciFourier ^{T}\circ \mathcal{H}\left( \mathbf{x}\right) \right) =%
\mathbb{E}\sup_{\mathbf{f}\in \tciFourier ^{T},h\in \mathcal{H}%
}\sum_{ti}\gamma _{ti}f_{t}\left( h\left( x_{ti}\right) \right) .
\end{equation*}

For a definite example we take $\mathcal{H}$ and $\tciFourier $ as in the
previous section and observe that now there is an additional factor $T$ on
the sample size. This implies the modified bounds $G\left( \mathcal{H}\left( 
\mathbf{x}\right) \right) \leq B_{1}\sqrt{Tnm_{1}}$ and $D\left( \mathcal{H}%
\left( \mathbf{x}\right) \right) \leq 2B_{1}\sqrt{Tn}$. Also for $\mathbf{y},%
\mathbf{y}^{\prime }\in 
\mathbb{R}
^{Tnm_{1}}$ with $y_{ti},y_{ti}^{\prime }\in 
\mathbb{R}
^{m_{1}}$ 
\begin{eqnarray*}
\sup_{\mathbf{f}\in \tciFourier ^{T}}\sum_{ti}\left( f_{t}\left(
y_{ti}\right) -f_{t}\left( y_{ti}^{\prime }\right) \right) ^{2} &\leq
&\sum_{t}\sup_{f\in \tciFourier }\sum_{i}\left( f_{t}\left( y_{ti}\right)
-f_{t}\left( y_{ti}^{\prime }\right) \right) ^{2} \\
&\leq &L^{2}\left( \tciFourier ,%
\mathbb{R}
^{nm_{1}}\right) \sum_{t}\sum_{i}\left\Vert y_{ti}-y_{ti}^{\prime
}\right\Vert ^{2},
\end{eqnarray*}%
so%
\begin{equation}
L\left( \tciFourier ^{T},%
\mathbb{R}
^{Tnm_{1}}\right) =L\left( \tciFourier ,%
\mathbb{R}
^{nm_{1}}\right) .  \label{Multitask rule for L}
\end{equation}%
Therefore $L\left( \tciFourier ^{T},%
\mathbb{R}
^{Tnm_{1}}\right) \leq B_{2}\Delta _{2}^{-1}$. Similarly 
\begin{eqnarray*}
&&\mathbb{E}\sup_{\mathbf{f}\in \tciFourier ^{T}}\sum_{ti}\gamma _{ti}\left(
f_{t}\left( y_{ti}\right) -f_{t}\left( y_{ti}^{\prime }\right) \right) \\
&=&\sum_{t}\mathbb{E}\sup_{f\in \tciFourier }\sum_{i}\gamma _{ti}\left(
f_{t}\left( y_{ti}\right) -f_{t}\left( y_{ti}^{\prime }\right) \right) \\
&\leq &\sqrt{T}\left( \sum_{t}\left( \mathbb{E}\sup_{f\in \tciFourier
}\sum_{i}\gamma _{ti}\left( f_{t}\left( y_{ti}\right) -f_{t}\left(
y_{ti}^{\prime }\right) \right) \right) ^{2}\right) ^{1/2} \\
&\leq &\sqrt{T}\left( \sum_{t}R^{2}\left( \tciFourier ,%
\mathbb{R}
^{nm_{1}}\right) \sum_{i}\left\Vert y_{ti}-y_{ti}^{\prime }\right\Vert
^{2}\right) ^{1/2} \\
&=&\sqrt{T}R\left( \tciFourier ,%
\mathbb{R}
^{nm_{1}}\right) \left\Vert \mathbf{y}-\mathbf{y}^{\prime }\right\Vert .
\end{eqnarray*}%
We conclude that%
\begin{equation}
R\left( \tciFourier ^{T},%
\mathbb{R}
^{nmT}\right) \leq \sqrt{T}R\left( \tciFourier ,%
\mathbb{R}
^{nm}\right) ,  \label{Multitask rule for R}
\end{equation}%
in the given case%
\begin{equation*}
R\left( \tciFourier ^{T},%
\mathbb{R}
^{nmT}\right) \leq \sqrt{T}B_{2}\Delta _{2}^{-1}.
\end{equation*}%
Also%
\begin{eqnarray}
G\left( \tciFourier ^{T}\left( h_{0}\left( \mathbf{x}\right) \right) \right)
&=&\mathbb{E}\sup_{\mathbf{f}\in \tciFourier ^{T}}\sum_{ti}\gamma
_{ti}f_{t}\left( h_{0}\left( x_{ti}\right) \right)  \notag \\
&=&\sum_{t}\mathbb{E}\sup_{f\in \tciFourier }\sum_{i}\gamma _{ti}f\left(
h_{0}\left( x_{ti}\right) \right)  \notag \\
&\leq &TG\left( \tciFourier \left( h_{0}\left( \mathbf{x}\right) \right)
\right) ,  \label{G_h_0 bound}
\end{eqnarray}%
so that here $G\left( \tciFourier ^{T}\left( h_{0}\left( \mathbf{x}\right)
\right) \right) \leq B_{2}T\sqrt{n}$. The chain rule then gives%
\begin{equation*}
G\left( \tciFourier \circ \mathcal{H}\left( \mathbf{x}\right) \right) \leq
C_{1}B_{1}B_{2}\Delta _{2}^{-1}\sqrt{Tnm_{1}}+\left( 2C_{2}B_{1}\Delta
_{2}^{-1}+1\right) B_{2}T\sqrt{n},
\end{equation*}%
where the first term represents the complexity of $\mathcal{H}$ and the
second that of $\tciFourier ^{T}$. Dividing by $nT$ we obtain as the
dominant term for the estimation error%
\begin{equation*}
C_{1}B_{1}B_{2}\Delta _{2}^{-1}\sqrt{\frac{m_{1}}{nT}}+\frac{\left(
2C_{2}B_{1}\Delta _{2}^{-1}+1\right) B_{2}}{\sqrt{n}}.
\end{equation*}%
This reproduces a general property of multitask learning \cite{Baxter 2000}:
in the limit $T\rightarrow \infty $ the contribution of the common
representation (including the intermediate dimension $m_{1}$) to the
estimation error vanishes. There remains only the cost of estimating the
task specific functions in the top layer.

We have obtained this result for a very specific model. The relations (\ref%
{Multitask rule for L}), (\ref{Multitask rule for R}) and (\ref{G_h_0 bound}%
) for $L\left( \tciFourier ^{T}\right) $, $R\left( \tciFourier ^{T}\right) $
and $G\left( \tciFourier ^{T}\left( h_{0}\left( \mathbf{x}\right) \right)
\right) $ are nevertheless independent of the exact model, so the chain rule
could be made the basis of a fairly general result about multitask feature
learning.

\subsection{Iteration of the bound}

We apply the chain rule to multi-layered or "deep" learning machines, a
subject which appears to be of some current interest. Here we have function
classes $\tciFourier _{1},...,\tciFourier _{K}$, where $\tciFourier _{k}$
consists of functions $f:%
\mathbb{R}
^{n_{k-1}}\rightarrow 
\mathbb{R}
^{n_{k}}$ and we are interested in the generalization properties of the
composite class 
\begin{equation*}
\tciFourier _{K}\circ ...\circ \tciFourier _{1}=\left\{ \mathbf{x}\in 
\mathbb{R}
^{n_{0}}\mapsto f_{K}\left( f_{K-1}\left( ...\left( f_{1}\left( \mathbf{x}%
\right) \right) \right) \right) :f_{k}\in \tciFourier _{k}\right\} .
\end{equation*}%
To state our bound we are given some sample $\mathbf{x}$ in $%
\mathbb{R}
^{n_{0}}$ and introduce the notation%
\begin{eqnarray*}
Y_{0} &=&\mathbf{x} \\
Y_{k} &=&\tciFourier _{k}\left( Y_{k-1}\right) =\tciFourier _{k}\circ
...\circ \tciFourier _{1}\left( \mathbf{x}\right) \subseteq 
\mathbb{R}
^{n_{k}},\text{ for }k>0 \\
G_{k} &=&\min_{\mathbf{y}\in Y_{k-1}\left( \mathbf{x}\right) }G\left(
\tciFourier _{k}\left( \mathbf{y}\right) \right) .
\end{eqnarray*}%
Under the convention that the product over an empty index set is $1$,
induction shows that 
\begin{equation*}
G\left( Y_{K}\right) \leq \sum_{k=1}^{K}\left(
C_{1}^{K-k}\prod_{j=k+1}^{K}L\left( \tciFourier _{j}\right) \right) \left(
C_{2}D\left( Y_{k-1}\right) R\left( \tciFourier _{k}\right) +G_{k}\right) .
\end{equation*}%
Clearly the large constants are prohibitive for any useful quantitative
prediction of generalization, but qualitative statements are possible.
Observe for example that, apart from $C_{1}$ and the Lipschitz constants,
each layer only makes an additive contribution to the overall complexity.
More specifically, for machine learning with a sample of size $n$, we can
make the assumptions $n_{k}=nm_{k}$, where $m_{k}$ is the dimension of the $%
k $-th intermediate representations, and it is reasonable to postulate $\max
\left\{ G_{k},D\left( Y_{k}\right) R\left( \tciFourier _{k}\right) \right\}
\leq Cn^{p}$, where $C$ is some constant not depending on $n$ and $p$ is
some exponent $p<1$ (for multi-layered kernel machines with Lipschitz
feature maps we would have $p=1/2$ - see above). Then the above expression
is of order $n^{p}$ and Theorem \ref{Theorem Rademacher bound} yields a
uniform law of large numbers for the multi-layered class, with a uniform
bound on the estimation error decreasing as $n^{p-1}$.

\section{Proof of Theorem \protect\ref{Theorem generic chaining}\label%
{Section proof of generic chaining}}

Talagrand has proved the following result (\cite{Talagrand 1992}).

\begin{theorem}
\label{Theorem majorizing measure}There are universal constants $r\geq 2$
and $C$ such that for every finite $Y\subset 
\mathbb{R}
^{n}$ there is an increasing sequence of partitions $\mathcal{A}_{k}$ of $Y$
and a probability measure $\mu $ on $Y$, such that, whenever $A\in \mathcal{A%
}_{k}$ then $D\left( A\right) \leq 2r^{-k}$ and%
\begin{equation*}
\sup_{\mathbf{y}\in Y}\sum_{k>k_{0}}^{\infty }r^{-k}\sqrt{\ln \frac{1}{\mu
\left( A_{k}\left( \mathbf{y}\right) \right) }}\leq C~G\left( Y\right) ,
\end{equation*}%
where $A_{k}\left( \mathbf{y}\right) $ denotes the unique member of $%
\mathcal{A}_{k}$ which contains $\mathbf{y}$, and $k_{0}$ is the largest
integer $k$ satisfying%
\begin{equation*}
2r^{-k}\geq D\left( Y\right) =\sup_{\mathbf{y},\mathbf{y}^{\prime }\in
Y}\left\Vert \mathbf{y}-\mathbf{y}^{\prime }\right\Vert
\end{equation*}
\end{theorem}

Observe that $2r^{-k_{0}}\geq D\left( Y\right) $, so we can assume $\mathcal{%
A}_{k_{0}}=\left\{ Y\right\} $. As explained in \cite{Talagrand 1992}, the
above Theorem is equivalent to the existence of a measure $\mu $ on $Y$ such
that 
\begin{equation*}
\sup_{\mathbf{y}\in Y}\int_{0}^{\infty }\sqrt{\ln \frac{1}{\mu \left(
B\left( \mathbf{y}\text{,}\epsilon \right) \right) }}d\epsilon \leq
C~G\left( Y\right) ,
\end{equation*}%
where $C$ is some other universal constant and $B\left( \mathbf{y}\text{,}%
\epsilon \right) $ is the ball of radius $\epsilon $ centered at $\mathbf{y}$%
. The latter is perhaps the more usual formulation of the majorizing measure
theorem.

We will use Talagrand's theorem to prove Theorem \ref{Theorem generic
chaining}, but before please note the inequality 
\begin{equation}
D\left( Y\right) \leq \sqrt{2\pi }G\left( Y\right) ,  \label{D-G inequality}
\end{equation}%
which follows from 
\begin{eqnarray*}
\sup_{\mathbf{y},\mathbf{y}^{\prime }\in Y}\left\Vert \mathbf{y}-\mathbf{y}%
^{\prime }\right\Vert &=&\sqrt{\frac{\pi }{2}}\sup_{\mathbf{y},\mathbf{y}%
^{\prime }}{{\mathbb{E}}}\left\vert \left\langle \mathbf{\gamma },\mathbf{y}-%
\mathbf{y}^{\prime }\right\rangle \right\vert \\
&\leq &\sqrt{\frac{\pi }{2}}{{\mathbb{E}}}\sup_{\mathbf{y},\mathbf{y}%
^{\prime }}\left\vert \left\langle \mathbf{\gamma },\mathbf{y}-\mathbf{y}%
^{\prime }\right\rangle \right\vert =\sqrt{\frac{\pi }{2}}{{\mathbb{E}}}%
\sup_{\mathbf{y},\mathbf{y}^{\prime }}\left\langle \mathbf{\gamma },\mathbf{y%
}-\mathbf{y}^{\prime }\right\rangle .
\end{eqnarray*}%
In the first equality we used the fact that $\left\Vert v\right\Vert =\sqrt{%
\pi /2}{{\mathbb{E}}}\left\vert \left\langle \mathbf{\gamma },v\right\rangle
\right\vert $ for any vector $v$.\bigskip

\begin{proof}[of Theorem \protect\ref{Theorem generic chaining}.]
Let $\mu $ and $\mathcal{A}_{k}$ be as determined for $Y$ by Theorem \ref%
{Theorem majorizing measure}. First we claim that for any $\delta \in \left(
0,1\right) $%
\begin{equation}
\Pr \left\{ \exists \mathbf{y}\in Y:X_{\mathbf{y}}-X_{\mathbf{y}%
_{0}}>\sum_{k>k_{0}}r^{-k+1}\sqrt{8\ln \left( \frac{2^{k-k_{0}}K}{\mu \left(
A\left( \mathbf{y}\right) \right) \delta }\right) }\right\} <\delta .
\label{claim}
\end{equation}%
For every $k>k_{0}$ and every $A\in \mathcal{A}_{k}$ let $\pi \left(
A\right) $ be some element chosen from $A$. We set $\pi \left( Y\right) =%
\mathbf{y}_{0}$. We denote $\pi _{k}\left( \mathbf{y}\right) =\pi \left(
A_{k}\left( \mathbf{y}\right) \right) $. This implies the chaining identity:%
\begin{equation*}
X_{\mathbf{y}}-X_{\mathbf{y}_{0}}=\sum_{k>k_{0}}\left( X_{\pi _{k}\left( 
\mathbf{y}\right) }-X_{\pi _{k-1}\left( \mathbf{y}\right) }\right) \text{.}
\end{equation*}%
For $k>k_{0}$ and $A\in \mathcal{A}_{k}$ use $\hat{A}$ to denote the unique
member of $\mathcal{A}_{k-1}$ such that $A\subseteq \hat{A}$. Since for $%
A\in \mathcal{A}_{k}$ both $\pi \left( A\right) $ and $\pi \left( \hat{A}%
\right) $ are members of $\hat{A}\in \mathcal{A}_{k-1}$ we must have $%
\left\Vert \pi \left( A\right) -\pi \left( \hat{A}\right) \right\Vert \leq
2r^{-k+1}$. Also note $\pi _{k-1}\left( \mathbf{y}\right) =\pi \left( \hat{A}%
_{k}\left( \mathbf{y}\right) \right) =\pi \left( \left( A_{k}\left( \pi
_{k}\left( \mathbf{y}\right) \right) \right) \symbol{94}\right) $. For $%
k\geq k_{0}$ we define a function $\xi _{k}:\mathcal{A}_{k}\rightarrow 
\mathbb{R}
_{+}$ as follows: 
\begin{equation*}
\xi _{k}\left( A\right) =r^{-k+1}\sqrt{8\ln \left( \frac{2^{k-k_{0}}K}{\mu
\left( A\right) \delta }\right) }.
\end{equation*}%
To prove the claim we have to show that%
\begin{equation*}
\Pr \left\{ \exists \mathbf{y}\in Y:X_{\mathbf{y}}-X_{\mathbf{y}%
_{0}}-\sum_{k>k_{0}}\xi _{k}\left( A_{k}\left( \mathbf{y}\right) \right)
>0\right\} <\delta .
\end{equation*}%
Denote the left hand side of this inequality with $P$. By the chaining
identity%
\begin{equation*}
P\leq \Pr \left\{ \exists \mathbf{y}:\sum_{k>k_{0}}\left( X_{\pi _{k}\left( 
\mathbf{y}\right) }-X_{\pi _{k-1}\left( \mathbf{y}\right) }-\xi _{k}\left(
A_{k}\left( \mathbf{y}\right) \right) \right) >0\right\} .
\end{equation*}%
If the sum is positive, at least one of the terms has to be positive, so%
\begin{equation*}
P\leq \Pr \left\{ \exists \mathbf{y},k>k_{0}:\left( X_{\pi _{k}\left( 
\mathbf{y}\right) }-X_{\pi _{k-1}\left( \mathbf{y}\right) }-\xi _{k}\left(
A_{k}\left( \mathbf{y}\right) \right) \right) >0\right\} .
\end{equation*}%
The event on the right hand side can also be written as%
\begin{equation*}
\left\{ \exists k>k_{0},\exists A\in \mathcal{A}_{k}:X_{\pi \left( A\right)
}-X_{\pi \left( \hat{A}\right) }>\xi _{k}\left( A\right) \right\} ,
\end{equation*}%
and a union bound gives%
\begin{eqnarray*}
P &\leq &\sum_{k>k_{0}}\sum_{A\in \mathcal{A}_{k}}\Pr \left\{ X_{\pi \left(
A\right) }-X_{\pi \left( \hat{A}\right) }>\xi _{k}\left( A\right) \right\} 
\\
&\leq &\sum_{k>k_{0}}\sum_{A\in \mathcal{A}_{k}}K\exp \left( \frac{-\xi
_{k}\left( A\right) ^{2}}{2\left\Vert \pi \left( A\right) -\pi \left( \hat{A}%
\right) \right\Vert ^{2}}\right)  \\
&\leq &\sum_{k>k_{0}}\sum_{A\in \mathcal{A}_{k}}K\exp \left( \frac{-\xi
_{k}\left( A\right) ^{2}}{2\left( 2r^{-k+1}\right) ^{2}}\right) ,
\end{eqnarray*}%
where we used the bound (\ref{Process Tailbound}) in the second and the
bound on $\left\Vert \pi \left( A\right) -\pi \left( \hat{A}\right)
\right\Vert $ in the third inequality. Using the definition of $\xi
_{k}\left( A\right) $ the last expression is equal to%
\begin{equation*}
\delta \sum_{k>k_{0}}\frac{1}{2^{k-k_{0}}}\sum_{A\in \mathcal{A}_{k}}\mu
\left( A\right) =\delta \sum_{k>k_{0}}\frac{1}{2^{k-k_{0}}}=\delta ,
\end{equation*}%
because $\mu $ is a probability measure. This establishes the claim.

Now, using $\sqrt{a+b}\leq \sqrt{a}+\sqrt{b}$ for $a,b\geq 0$, with
probability at least $1-\delta $%
\begin{eqnarray*}
\sup_{\mathbf{y}}X_{\mathbf{y}}-X_{\mathbf{y}_{0}} &\leq
&r\sum_{k>k_{0}}r^{-k}\sqrt{8\ln \left( \frac{1}{\mu \left( A_{k}\left( 
\mathbf{y}\right) \right) }\right) } \\
&&\text{ \ \ \ \ \ \ \ }+r^{-k_{0}+1}\sum_{k>0}r^{-k+1}\sqrt{8\ln \left( 
\frac{2^{k}K}{\delta }\right) } \\
&\leq &\sqrt{8}rC~G\left( Y\right) +\sqrt{8}r^{-k_{0}+1}\sum_{k>0}\sqrt{k}%
r^{-k+1}\sqrt{\ln \left( \frac{2K}{\delta }\right) },
\end{eqnarray*}%
where we used Talagrand's theorem and the fact that $K>1$. By the definition
of $k_{0}$ we have $r^{-k_{0}+1}\leq r^{2}D\left( Y\right) /2$, so this is
bounded by%
\begin{equation*}
C^{\prime \prime \prime }G\left( Y\right) +C^{\prime \prime \prime \prime
}D\left( Y\right) \sqrt{\ln \left( \frac{2K}{\delta }\right) },
\end{equation*}%
with $C^{\prime \prime \prime }=\sqrt{8}rC$ and $C^{\prime \prime \prime
\prime }=\sqrt{8}\left( r^{2}/2\right) \sum_{k>0}\sqrt{k}r^{-k+1}$.
Converting the last bound into a tail bound and integrating we obtain%
\begin{eqnarray*}
\mathbb{E}\left[ \sup_{\mathbf{y}}X_{\mathbf{y}}-X_{\mathbf{y}_{0}}\right] 
&\leq &C^{\prime \prime \prime }G\left( Y\right) +C^{\prime \prime \prime
\prime }D\left( Y\right) \left( \sqrt{\ln 2K}+\frac{\sqrt{\pi }}{2}\right) 
\\
&\leq &C^{\prime \prime \prime }G\left( Y\right) +3C^{\prime \prime \prime
\prime }D\left( Y\right) \sqrt{\ln 2K} \\
&\leq &\left( C^{\prime \prime \prime }+3\sqrt{2\pi \ln 2}C^{\prime \prime
\prime \prime }\right) G\left( Y\right) +3C^{\prime \prime \prime \prime
}D\left( Y\right) \sqrt{\ln K},
\end{eqnarray*}%
where we again used $K\geq 1$ in the second inequality and (\ref{D-G
inequality}) in the last inequality. This gives the conclusion with $%
C^{\prime }=C^{\prime \prime \prime }+3\sqrt{2\pi \ln 2}C^{\prime \prime
\prime \prime }$ and $C^{\prime \prime }=3C^{\prime \prime \prime \prime }$. 
$\blacksquare $\bigskip 
\end{proof}

\end{document}